\documentclass{l4dc2026}

\title[RL-MAPF]{Deadlock-Free Hybrid RL-MAPF Framework for Zero-Shot Multi-Robot Navigation}
\usepackage{times}
\usepackage{enumitem}
\usepackage[algo2e,ruled]{algorithm2e}
\usepackage{wrapfig}

\graphicspath{{figs/}} 
\LinesNumbered
\SetAlgoVlined
\DontPrintSemicolon

\author{%
 \Name{Haoyi Wang$^{2}$} \Email{haoyi@wustl.edu}
 \AND
 \Name{Licheng Luo$^{1}$} \Email{lichengl@ucr.edu}
  \AND
\Name{Yiannis Kantaros$^{2}$} \Email{ioannisk@wustl.edu}
  \AND
\Name{Bruno Sinopoli$^{2}$} \Email{bsinopoli@wustl.edu}
  \AND
  \Name{Mingyu Cai$^{1}$} \Email{mingyu.cai@ucr.edu}\\
 \addr $^{1}$Department of Mechanical Engineering, University of California Riverside, CA, USA\\
 \addr $^{2}$Department of Electrical and Systems Engineering, Washington University in St. Louis, MO, USA
}

\begin{document}

\maketitle

\begin{abstract}%
 Multi-robot navigation in cluttered environments presents fundamental challenges in balancing reactive collision avoidance with long-range goal achievement. When navigating through narrow passages
 or confined spaces, deadlocks frequently emerge that prevent agents from reaching their destinations, particularly when Reinforcement Learning (RL) control policies encounter novel configurations out of learning distribution. Existing RL-based approaches suffer from limited generalization capability in unseen environments. We propose a hybrid framework that seamlessly integrates RL-based reactive navigation with on-demand Multi-Agent Path Finding (MAPF) to explicitly resolve topological deadlocks. Our approach integrates a safety layer that monitors agent progress to detect deadlocks and, when detected, triggers a coordination controller for affected agents. The framework constructs globally feasible trajectories via MAPF and regulates waypoint progression to reduce inter-agent conflicts during navigation.
 Extensive evaluation on dense multi-agent benchmarks shows that our method boosts task completion from marginal to near-universal success, markedly reducing deadlocks and collisions. When integrated with hierarchical task planning, it enables coordinated navigation for heterogeneous robots, demonstrating that coupling reactive RL navigation with selective MAPF intervention yields a robust, zero-shot performance. Additional details are available at \url{https://wanghaoyi518.github.io/rl-mapf-project-page/}.
 
\end{abstract}


\section{Introduction}

\emph{Multi-agent long-range navigation} in obstacle-rich environments requires both global guidance and reliable local interaction handling. Classical geometric local planners such as Velocity Obstacles (VO)/ Reciprocal Velocity Obstacles (RVO) / Optimal Reciprocal Collision Avoidance (ORCA) offer fast, distributed collision avoidance \citep{fiorini1998vo,snape2011hrvo,vandenberg2011orca,alonsomora2012borca}, while deep RL controllers leverage high-dimensional observations for reactive goal pursuit and short-horizon interactions \citep{everett2018ga3c,chen2017cadrl,chen2019sarl,han2022rlrvo}. For long-range tasks, a common practice is to generate a coarse \emph{global path} (e.g., A* or Theta*) and let the RL policy \emph{track a sequence of waypoints} along that path, combining topological guidance with learned local navigation \citep{hart1968astar,nash2007theta,nash2010lazytheta,yakovlev2017aasipp, cai2023overcoming, luo2025bridging}. This “global-waypoints + local RL” pattern has been widely adopted to bridge large-scale routing with dynamics-aware local control. However, RL-based distributed navigation \emph{alone} often struggles to generalize in complex, unseen layouts; the space of environmental patterns is vast, and purely learned policies can deteriorate out-of-distribution. Even under the waypoint-tracking setup, \emph{multi-agent} deployments introduce additional failure modes: independently generated waypoint lists can place agents on \emph{mutually proximate} targets or corridors, triggering persistent reciprocal avoidance that prevents agents from reaching waypoints. To define such mechanisms, We say a \emph{deadlock} occurs (see Fig.~\ref{fig:deadlock}) when a subset of agents, under purely reciprocal/local policies, enters a persistent standstill at a bottleneck and cannot make goal-directed progress \citep{dergachev2021distributed,grover2019deadlock,wang2008fast}.

\begin{figure}[t]
  \centering
  \includegraphics[width=0.7\linewidth]{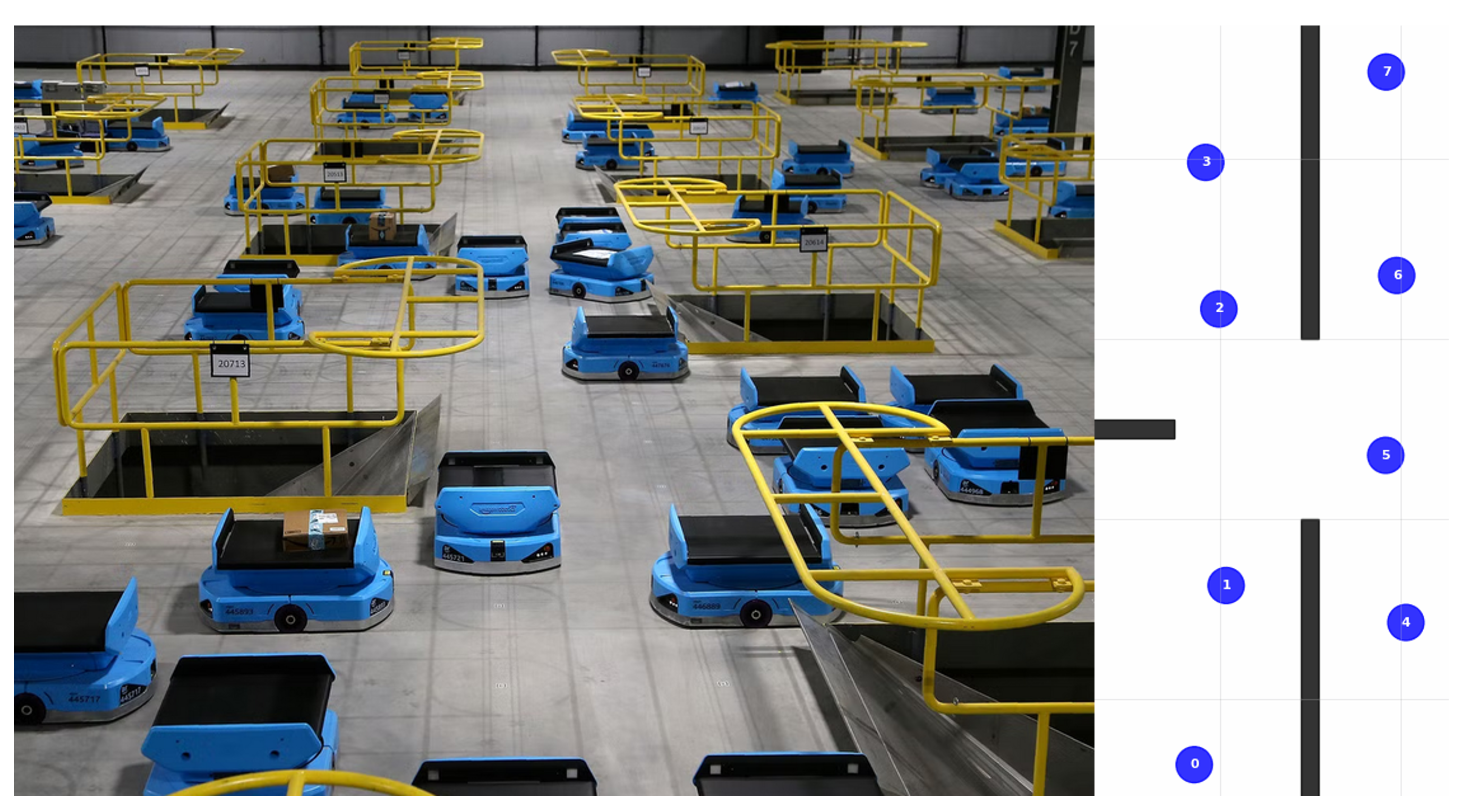}
  \caption{Deadlock phenomenon in practice and simulation. Reciprocal avoidance can saturate at bottlenecks, leading to deadlocks (photo adapted from \citet{bui2023surveyMRMP}). 
  }

  \label{fig:deadlock}
\end{figure}

On the other hand, MAPF methods can compute coordinated, collision-free paths with near-optimality guarantees on discretized graphs \citep{sharon2015cbs,barer2014ecbs,li2021eecbs,dewilde2014pnr,surynek2021gcbc}, but incur non-trivial computational cost at scale and—crucially—typically assume agents can \emph{faithfully follow} planned paths under simplified kinematic or time models, which may be brittle in practice; these issues motivate using global plans for guidance while avoiding always-on team-wide coordination \citep{stern2019mapfsurvey,vandenberg2011orca,ma2021lifelong,li2021mapf_lns, ge2025efficient, huang2024impact}.

We propose a \emph{triggered hybrid} that preserves decentralized efficiency almost everywhere while injecting \emph{locally confined coordination} only when necessary. By default, robots run an \emph{RL mode} that tracks a global waypoint list and performs reciprocal avoidance. When a lightweight \emph{deadlock detector} flags a persistent standstill around a bottleneck (e.g., a doorway swap), we crop a local grid over the implicated agents and solve a small MAPF instance. For each such agent, the solver returns a short \emph{dense waypoint list}—a time-ordered sequence of reference waypoints. \emph{Crucially, the low-level controller remains the same RL policy}, which now tracks these dense waypoints while handling local avoidance; non-implicated agents continue in the default RL mode. After the dense lists are executed (the stalemate cleared), agents seamlessly return to tracking their original global waypoints. The local subplans rely on grid-based coordination primitives for quick, topology-aware disentanglement with bounded spatial/temporal scope \citep{dewilde2014pnr,barer2014ecbs,li2021eecbs}.

The key is to \emph{separate regimes by topology and time scale}. (i) In easy regimes with ample free space, reactive RL/ORCA-like behavior suffices for progress and safety \citep{vandenberg2011orca,alonsomora2012borca}. (ii) In permutation-constrained regimes (doorways, swaps, narrow loops) where deadlocks are detected, a brief, \emph{locally confined} MAPF solve produces for only the implicated agents a short \emph{dense waypoint list} that is topologically valid for disentanglement \citep{dewilde2014pnr,sharon2015cbs}. Because the \emph{low-level controller remains the same RL policy}, these dense waypoints act as short-horizon references while RL continues to handle reciprocal avoidance and dynamics; once executed, agents revert to tracking their original global waypoints. Thus coordination is \emph{on demand and local}, keeping the amortized overhead small while reliably breaking stalemates by reordering a few agents rather than coordinating the whole team \citep{dergachev2020combo,dergachev2021distributed,stern2019mapfsurvey}.
The contributions are as follows: (i) A \emph{deadlock-triggered, locally confined MAPF} module that augments a decentralized RL navigator only when needed; (ii) a practical pipeline that unifies reactive waypoint tracking with brief, topology-aware reordering, keeping computation and intervention \emph{localized}; (iii) empirical evidence across dense exchanges and narrow passages showing high success, compared against strong baseline \citep{han2022rlrvo}. Additional materials can be found on the project website \url{https://wanghaoyi518.github.io/rl-mapf-project-page/}.

\section{Related Work}

\paragraph{RL-based multi-robot navigation.}
Deep reinforcement learning (DRL) has emerged as a powerful paradigm for interactive, decentralized collision avoidance, where a policy maps onboard observations of nearby agents/obstacles to continuous controls with no explicit communication. Early work such as CADRL/GA3C-CADRL learned value or policy networks that anticipate reciprocal interactions among multiple decision-making agents, later incorporating attention to model higher-order human–robot and human–human influences in dense crowds \citep{chen2017cadrl,everett2018ga3c,chen2019sarl}. Sensor-level DRL for multi-robot collision avoidance and crowd navigation further improves generalization across layouts and team sizes \citep{long2018drlmaca,fan2020ijrr}. These approaches show strong short-horizon reactivity and scalability to a variable number of neighbors, but may suffer from myopic behaviors and limited long-range consistency, especially in cluttered layouts with bottlenecks and symmetry-induced stand-offs. Structure-aware formulations that encode velocity-obstacle geometry into the learning pipeline, e.g., using VO/RVO/ORCA vectors in the observation and shaping rewards with predicted collision time, improve safety and data efficiency while retaining decentralized execution \citep{fiorini1998vo,vandenberg2011orca,snape2011hrvo,alonsomora2012borca,han2022rlrvo}. Classical highly parallel local planners (e.g., ClearPath) underscore the efficiency of VO-style formulations but share the same deadlock risks under reciprocity \citep{guy2009clearpath}. Our RL navigation policy follows this line by leveraging RVO-shaped state/reward design yet targets \emph{long-range} settings where purely reactive reciprocity is insufficient.

\paragraph{MAPF and hybrid navigation.}
MAPF plans joint, collision-free trajectories on discrete abstractions. Optimal and bounded-suboptimal variants of Conflict-Based Search provide strong solution-quality guarantees, while complete polynomial-time schemes such as Push-and-Rotate ensure feasibility under sufficient free space \citep{sharon2015cbs,barer2014ecbs,li2021eecbs,dewilde2014pnr,boyarski2015icbs}. For geometric guidance over long distances, global planners like A* supply compact waypoint sequences. MAPF has continued to scale via anytime/meta-heuristics (MAPF-LNS) and analyses of solution-quality vs.\ team size \citep{li2021mapflns,atzmon2020increasingcost}, with extensions to continuous time important for bridging to execution \citep{andreychuk2022ctmapf}. Benchmarks and surveys consolidate definitions, variants, and standardized evaluation (e.g., MovingAI) \citep{stern2019mapfsurvey,stern2019benchmarks}. 

Purely reactive stacks (e.g., ORCA) are computationally efficient and provably collision-free under reciprocity assumptions but are prone to deadlocks in narrow passages; hybrid pipelines therefore combine global/any-angle waypoints with reciprocal avoidance during execution and \emph{escalate} to locally confined MAPF when detectors flag prospective topological deadlocks \citep{vandenberg2011orca,dergachev2020combo,dergachev2021distributed}. Empirically, integrating Push-and-Rotate or (E)ECBS as an on-demand subsolver within an ORCA/Theta* stack boosts success dramatically in tight environments; we build on this architecture but replace hand-crafted local controllers with an RVO-shaped RL policy, preserving fast decentralized response while invoking MAPF only when necessary to re-order agents and restore flow \citep{dewilde2014pnr,barer2014ecbs,li2021eecbs,dergachev2020combo,dergachev2021distributed}.

\section{Problem Formulation}
We consider a team of \( n \) circular agents \(\mathcal{A} = \{1, \ldots, n\}\) moving in a static two-dimensional workspace \( W \subset \mathbb{R}^2 \) that contains obstacles \( O \subset W \) and free space \( F = W \setminus O \). Each agent \( i \) has radius \( r_i > 0 \), state \( x_i(t) = (p_i(t), v_i(t)) \), position \( p_i(t) \in F \), velocity \( v_i(t) \in \mathbb{R}^2 \), and speed limit \( \|v_i(t)\| \le v_i^{\max} \), where time is discrete with step size \(\Delta t > 0\). Agents sense and communicate within a fixed range \( R > 0 \); the neighbor set of agent \( i \) at time \( t \) is defined as \(\mathcal{N}_i(t) = \{ j \in \mathcal{A} \setminus \{i\} \mid \|p_j(t) - p_i(t)\| \le R \}\), and agents exchange only the local information necessary for reciprocal avoidance and deadlock resolution. Each agent \( i \) is assigned a start position \( s_i \in F \) and a goal position \( g_i \in F \), and the objective is to drive all agents from their respective start to goal positions. 

A joint trajectory \(\Pi = \{\pi_1, \ldots, \pi_n\}\) with \(\pi_i = \{p_i(t)\}_{t \ge 0}\) is considered collision-free if, for all \(t\) and all \(i \neq j\), (i) \(p_i(t) \in F\) and (ii) \(\|p_i(t) - p_j(t)\| > r_i + r_j\); in practice, each agent selects instantaneous velocities that satisfy reciprocal collision-avoidance constraints with both neighbors and obstacles. A run is deemed successful if there exists a finite time \(T\) such that \(p_i(T) = g_i\) for all \(i\) while maintaining safety for all \(t \le T\); the team success rate is reported in experiments. At each time step, each agent applies a decentralized hybrid control policy \(\pi_{\mathrm{RL}}\) that uses local observations to generate a commanded velocity driving progress toward the goal while avoiding collisions with nearby agents and obstacles.
When local progress stalls or reciprocal stand-offs occur, a deadlock is detected using the predicate:
\begin{equation}
    \mathsf{deadlock}(t) = \mathrm{true} \quad\text{if}\quad
    \underbrace{\frac{1}{k}\sum_{\tau=t-k+1}^{t} \mathrm{prog}_i(\tau)}_{\text{ego progress}} < \varepsilon_p
    \;\wedge\;
    \underbrace{\sum_{j \in \mathcal{N}_i(t)} \mathbf{1}\{\mathrm{prog}_j^{(k)} < \varepsilon_p\}}_{\text{neighbors stalled}} \ge \varepsilon_n.
\end{equation}
Detection of a deadlock triggers the formation of a locally confined MAPF instance over a grid subregion \( G' \subset F \) encompassing the involved agents and relevant obstacles. The subproblem uses the agents' current positions as starts and local waypoints as goals. A MAPF solver then computes a short-horizon joint plan for the affected agents, which they execute synchronously before returning to the decentralized policy \(\pi_{\mathrm{RL}}\).  
For long-range navigation, agents may optionally follow sparse guidance waypoints generated by a global planner on \( F \) (e.g., A*), serving as intermediate targets for \(\pi_{\mathrm{RL}}\) between MAPF interventions.

\noindent\textbf{Objective.} 
Develop an integrated framework combining the decentralized policy \(\pi_{\mathrm{RL}}\), the deadlock predicate, subregion construction \(G'\), and the MAPF solver \(\mathcal{P}\) to maximize team success in cluttered, previously unseen environments, while minimizing coordination overhead and maintaining the efficiency and responsiveness of \(\pi_{\mathrm{RL}}\) during normal operation.

\section{Methodology}

\begin{wrapfigure}{r}{0.6\textwidth}
    \vspace{-8pt}
    \centering
    \includegraphics[width=0.95\linewidth]{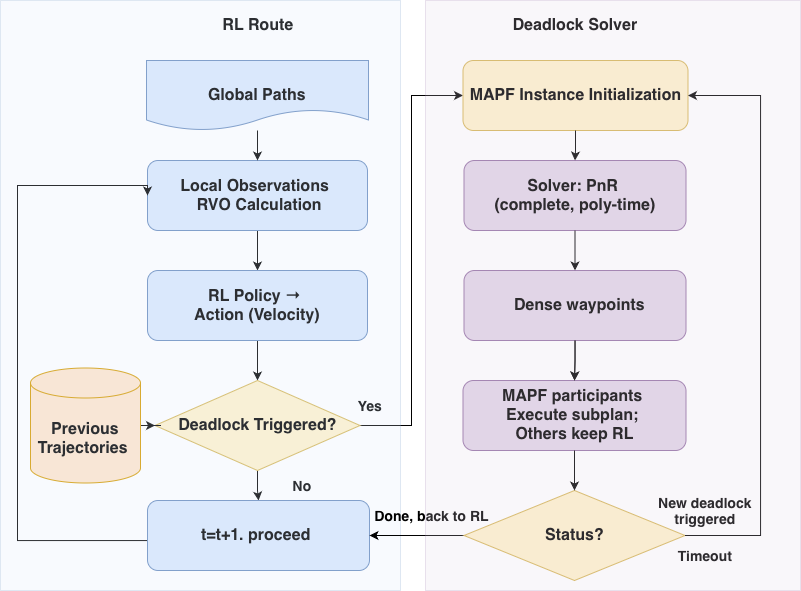}
    \caption{Hybrid RL+MAPF workflow.}
    \label{fig:rl-mapf-pipeline}
    \vspace{-8pt}
\end{wrapfigure}

We adopt the agent model, sensing assumptions, and VO/RVO formalism defined in Problem Formulation. Our controller is layered and switches \emph{only} a small, locally selected group when necessary; all other agents remain in decentralized RL. 
\textit{How the pieces connect:} \S 4.1 first produces per-agent \emph{global waypoint lists} on a grid to provide long-range guidance. \S4.2 details the \emph{RL navigation policy} that runs at every step for all agents, encoding VO/RVO-based neighbor features and outputting smooth velocity increments that track either the global waypoints (default) or temporary dense waypoints. \S4.3 monitors short-horizon progress and reciprocal interactions to \emph{trigger} coordination only when local stalemates are detected. Upon a trigger, \S4.4 forms a small \emph{locally confined MAPF} on a cropped subgrid (PnR / bounded-suboptimal CBS variants) and returns, for the implicated agents only, a short \emph{dense waypoint list}. These waypoints are executed by the \emph{same} RL low-level policy; once completed, agents \emph{seamlessly hand back} to the default RL navigation and continue tracking their original global waypoints. \S4.5 gives formal verification for the mathematical properties of our method. The overall workflow is summarized in Fig.~\ref{fig:rl-mapf-pipeline}.


\noindent\textbf{4.1. Global Path Planning (Waypoint Generation): }Each agent $i$ is assigned a long-range geometric \emph{guide} $\Gamma_i$ from start $p_{i,0}$ to goal $g_i$ on a known static map. We adopt a grid-based shortest-path planner (e.g., A*\citep{hart1968astar}) at a resolution shared with local planning to avoid frame/scale drift and then \emph{resample} the path into a waypoint sequence $\mathcal{W}_i=\{w_{i,k}\}_{k=0}^{K_i}$ (uniform or distance-threshold spacing). At runtime, the waypoint manager advances the active waypoint when within a reach threshold. 

\noindent\textbf{4.2. RL navigation policy: }
We use a decentralized learned local navigation policy that outputs smooth velocity \emph{increments} and runs on-board for each agent. Agents are discs of radius $R$ moving in $\mathcal{W}_{\text{free}}$ with step $\Delta t$; the ego agent $i$ has pose $p_{i,t}\!\in\!\mathbb{R}^2$ and velocity $v_{i,t}\!\in\!\mathbb{R}^2$, and we update $v_{i,t+1}=\operatorname{clip}\!\big(v_{i,t}+\mu\,\Delta v_{i,t},\|\,\cdot\,\|\le v_{\max}\big)$ and $p_{i,t+1}=p_{i,t}+v_{i,t+1}\Delta t$, where the policy outputs $\Delta v_{i,t}$ which is converted from $(v_x,v_y)$ to $(v,\omega)$ by a heading-alignment controller. 

\emph{Observations.} We encode ego–neighbor information using VO/RVO geometry \citep{fiorini1998vo,vandenberg2011orca,han2022rlrvo}. Let $\mathrm{ori}_{i,t}\!\in\!(-\pi,\pi]$ be heading, $g_i$ the goal, $w^{\text{curr}}_i$ the active global waypoint, and $v^{\text{des}}_{i,t}=k_p\,\frac{w^{\text{curr}}_i-p_{i,t}}{\|w^{\text{curr}}_i-p_{i,t}\|}$ with gain $k_p>0$. Neighbors within radius $R_{\text{sense}}$ form $\mathcal{N}_i(t)=\{\,j\neq i~|~\|p_{j,t}-p_{i,t}\|\le R_{\text{sense}}\,\}$. For each $j\in\mathcal{N}_i(t)$, we build a 6D descriptor $c_{ij}=(v^{\text{apex}}_{ij},\ell^L_{ij},\ell^R_{ij})\in\mathbb{R}^6$ with apex $v^{\text{apex}}_{ij}=v_{j,t}$ and VO boundary rays $\ell^{L/R}_{ij}$ induced by relative position $p_{j,t}-p_{i,t}$ and combined radius $R_c=2R$; we also record distance $d_{ij}=\|p_{j,t}-p_{i,t}\|$. The time-to-collision $t^{(e)}_{ij}$ is the smallest $\tau>0$ such that $v_{i,t}-v_{j,t}\in\mathrm{VO}_{ij}(\tau)$ (set to $+\infty$ if no collision is predicted), and we map it to a scalar risk $r^{(e)}_{ij}=1/(t^{(e)}_{ij}+0.2)$ to emphasize imminence. The policy input is $o^{\text{self}}_{i,t}=[v_{i,t},\,\mathrm{ori}_{i,t},\,v^{\text{des}}_{i,t},\,R_c]$ and $o^{\text{sur}}_{i,t}=\{(c_{ij},\,d_{ij},\,r^{(e)}_{ij})\}_{j\in\mathcal{N}_i(t)}$.

\emph{Network \& training.} We aggregate the variable-size set $\{(c_{ij},d_{ij},r^{(e)}_{ij})\}$ with a BiGRU encoder, concatenate with $o^{\text{self}}_{i,t}$, and feed actor–critic heads trained by PPO \citep{schulman2017ppo}. 

\emph{Reward.} To promote goal progress and reciprocal safety \citep{han2022rlrvo}, we use 
$r_t=\alpha\big(\|p_{i,t}-g_i\|-\|p_{i,t+1}-g_i\|\big)-\beta\,\Phi_{\mathrm{RVO}}(\{c_{ij}\})-\gamma\,\psi\big(\min_{j\in\mathcal{N}_i(t)} t^{(e)}_{ij}\big)-\eta\,\mathbf{1}\{\text{collision}\}$, 
where $\alpha,\beta,\gamma,\eta>0$ are weights; $\Phi_{\mathrm{RVO}}$ is a smooth, nonnegative VO/RVO penetration measure in velocity space (e.g., sum of hinge distances of $v_{i,t}$ to ORCA half-planes), and $\psi(\cdot)$ penalizes short $t^{(e)}$ (e.g., hinge/reciprocal), both vanishing when constraints are satisfied.

\noindent\textbf{4.3. Deadlock Detection and RL$\rightarrow$MAPF Switching Trigger: }Purely reactive, decentralized avoidance (e.g., rvo-style reciprocity) is efficient in open space but can stall in tight passages where agents must be \emph{re-ordered} rather than merely sidestepping \citep{vandenberg2011orca,alonsomora2012borca}. 
This module monitors local kinematics to \emph{detect} such stalemates early and \emph{escalates only a minimal set} of implicated agents to a short, locally confined MAPF subproblem—leaving all others in the RL mode. 
This on-demand, localized coordination has been shown to reliably resolve doorway/corridor stand-offs while keeping amortized overhead low \citep{dergachev2021distributed,dergachev2020combo}.

\noindent\textbf{Design overview (triggers and roles).}
We employ three complementary triggers, each targeting a distinct deadlock symptom. 
(i) \emph{Speed / non-progress} detects \emph{mutual slowing without goal-directed motion} inside a neighborhood—an operational sign of reciprocal stand-off. 
(ii) \emph{Waypoint-stuck} detects \emph{long-horizon stagnation} by monitoring whether the active waypoint index advances under the global guidance protocol. 
(iii) \emph{Core-pair risk} identifies an \emph{imminent, reciprocally coupled pair} via short-horizon Time-to-Collision (TTC) / min-distance, which seeds the minimal coordination set. 
Their \emph{union} yields a fast, low-noise detector that triggers locally confined MAPF only when re-ordering is necessary; non-implicated agents remain in decentralized RL.

Let $p_i(t),v_i(t)\in\mathbb{R}^2$ be the position and velocity of agent $i$ at step $t$.
Denote the active target (current waypoint or final goal) by $w^{\mathrm{curr}}_i(t)$ and the sensed neighbor set
$\mathcal{N}_i(t)=\{\,j\neq i:\|p_j(t)-p_i(t)\|\le R\,\}$ with radius $R$.
Define the $K$-step trailing mean speed
$\bar{s}_i(t)=\frac{1}{K}\sum_{\tau=t-K+1}^{t}\|v_i(\tau)\|$,
the unit goal direction
$\hat{d}_i(t)=\frac{w^{\mathrm{curr}}_i(t)-p_i(t)}{\|w^{\mathrm{curr}}_i(t)-p_i(t)\|}$ when $\|w^{\mathrm{curr}}_i(t)-p_i(t)\|>\varepsilon_{\mathrm{goal}}$,
and the instantaneous progress
$\mathrm{prog}_i(t)=\langle v_i(t),\hat{d}_i(t)\rangle$.
To avoid spurious triggers we apply gating:
$t\ge T_{\mathrm{warm}}$, $\|p_i(t)-w^{\mathrm{curr}}_i(t)\|>\varepsilon_{\mathrm{goal}}$, and
$t-t^{\mathrm{last}}_i\ge T_{\mathrm{cool}}$ (where $t^{\mathrm{last}}_i$ is the last detection time).

\paragraph{Speed / non-progress trigger.}
Flag a potential stalemate when an agent moves slowly while at least one neighbor is also slow and not making forward progress:
\begin{equation}
\label{eq:spd}
\mathsf{D}_{\mathrm{spd}}(i,t)\;\Leftrightarrow\;\Big(\bar{s}_i(t)<v_{\mathrm{low}}\Big)\ \wedge\ \Big(\exists j\in\mathcal{N}_i(t):\ \bar{s}_j(t)<v_{\mathrm{low}}\ \wedge\ \mathrm{prog}_j(t)\le 0\Big),
\end{equation}
with hyperparameters $v_{\mathrm{low}}$ (low-speed threshold) and $K$ (window length).

\paragraph{Waypoint-stuck trigger.}
If the active waypoint index $k_i(t)$ has not advanced for a budget $T_{\mathrm{wp}}$ while the agent remains outside the goal tolerance, raise a deadlock:
\begin{equation}
\label{eq:wp}
\mathsf{D}_{\mathrm{wp}}(i,t)\;\Leftrightarrow\;\Big(t-t^{\mathrm{wp}}_i(t)\ge T_{\mathrm{wp}}\Big)\ \wedge\ \Big(\|p_i(t)-w^{\mathrm{curr}}_i(t)\|>\varepsilon_{\mathrm{goal}}\Big),
\end{equation}
where $t^{\mathrm{wp}}_i(t)$ is the last time $k_i(\cdot)$ changed.

\paragraph{Core-pair risk trigger.}
Using a short-horizon reciprocal-avoidance proxy popular in the literature, including TTC and minimum distance \citep{vandenberg2011orca,alonsomora2012borca}, define
\[
r_{ij}=p_j-p_i,\qquad u_{ij}=v_j-v_i,\qquad
\mathrm{ttc}_{ij}=\begin{cases}
+\infty,& \|u_{ij}\|^2\le\epsilon\ \text{or}\ \langle r_{ij},u_{ij}\rangle\ge 0\\[2pt]
-\,\dfrac{\langle r_{ij},u_{ij}\rangle}{\|u_{ij}\|^2},& \text{otherwise}
\end{cases},
\]
and $\mathrm{dmin}_{ij}=\|r_{ij}\|$ if $\mathrm{ttc}_{ij}=+\infty$, else
$\mathrm{dmin}_{ij}=\big\|r_{ij}+\mathrm{ttc}_{ij}u_{ij}\big\|$.
Let $b(i)=\arg\min_{j\in\mathcal{N}_i(t)}\mathrm{ttc}_{ij}$ (ties broken deterministically).
A mutually “most-at-risk” pair with imminent interaction activates:
\begin{equation}
\label{eq:risk}
\mathsf{D}_{\mathrm{risk}}(i,t)\;\Leftrightarrow\;\Big(b(b(i))=i\Big)\ \wedge\ \Big(\mathrm{ttc}_{i\,b(i)}<\tau_{\mathrm{ttc}}\ \vee\ \mathrm{dmin}_{i\,b(i)}<\delta_{\mathrm{min}}\Big),
\end{equation}
with thresholds $\tau_{\mathrm{ttc}}$ and $\delta_{\mathrm{min}}$ (and small $\epsilon$ for numerical stability).

\paragraph{Union rule and switch logic.}
A deadlock is declared for agent $i$ whenever any trigger fires: $\mathsf{D}(i,t)\;=\;\mathsf{D}_{\mathrm{spd}}(i,t)\ \vee\ \mathsf{D}_{\mathrm{wp}}(i,t)\ \vee\ \mathsf{D}_{\mathrm{risk}}(i,t)$, subject to the gating above. From the seed agent, we form a \emph{minimal} participant set by (i) including the core pair when \eqref{eq:risk} holds and (ii) applying a tiny consensus closure within $\mathcal{N}_i(t)$ to cover the immediate stalemate.
The selected set is \emph{locked} for $T_{\mathrm{lock}}$ steps to prevent oscillations; only agents for which a valid locally confined MAPF subplan is found switch to the MAPF mode, and all others remain in decentralized RL. This localized, on-demand escalation mirrors prior “locally confined MAPF” hybrids that resolve doorway/corridor stand-offs without team-wide coordination.



\noindent\textbf{4.4. Local MAPF Instance: }The triggers in \S 4.3 produce a \emph{local, minimal} set of implicated agents and a detection time $t$. 
This subsection turns that signal into a concrete coordination problem by (i) freezing the participant set and identifying the \emph{coordination group}, (ii) cropping a subgrid $\mathcal{G}'$ around them, (iii) instantiating a \emph{locally confined MAPF} $\mathcal{I}'=\langle \mathcal{G}',\mathcal{A}_{\mathrm{lc}},S,T\rangle$ using current poses as starts and projected active waypoints as local goals, and (iv) solving $\mathcal{I}'$ quickly to obtain per-agent \emph{dense waypoint lists}. 
These lists are executed by the \emph{same decentralized RL policy} (others remain in RL), and once the dense segments finish, all agents revert to tracking their original global waypoints. 
This “detect→crop→coordinate→handover” design follows prior locally confined MAPF frameworks for doorway/corridor stand-offs and uses PnR for fast, complete reordering on the subgrid.
Upon a deadlock trigger, a set of nearby agents enters a locally confined coordination phase on a subgrid $\mathcal{G}'$.
For the deadlock-triggering agent $i$ at time $t$, we define the local neighbor set within sensing radius $R$ as
$\mathcal{N}_i(t)\;=\;\big\{\, j\in \mathcal{A}\setminus\{i\}\ \big|\ \|\,p_j(t)-p_i(t)\,\|_2 \le R \,\big\}, \mathcal{A}_{\mathrm{lc}}(t)=\{i\}\,\cup\,\mathcal{N}_i(t)$. 
Here, $\mathcal{A}$ is the finite index set of agents; $p_i(t)\in\mathbb{R}^2$ denotes the position of agent $i$ at discrete step $t\in\mathbb{N}$; $R>0$ is the sensing radius; and $\|\,\cdot\,\|_2$ is the Euclidean norm.

All the other agents remain in decentralized RL mode. We construct $\mathcal{G}'$ by padding the participants’ axis-aligned bounding box with margin $m>0$ and discretizing at resolution $h_g$. Starts $S(\cdot)$ are assigned by nearest-cell projection of current poses; local goals $T(\cdot)$ are obtained by projecting each agent’s active global waypoint into $\mathcal{G}'$.

\paragraph{Solver and execution.}
We solve the local MAPF instance $\mathcal{I}'=\langle \mathcal{G}'{=}(V',E'),\,\mathcal{A}_{\mathrm{lc}},\,S,\,T\rangle$ with \emph{Push-and-Rotate (PnR)} algorithm \citep{dewilde2014pnr}. The joint discrete plan for each participant is converted into a short \emph{dense waypoint list} (time-ordered references) and tracked by the \emph{same} RL low-level controller under a mild speed cap; non-participants continue with RL. After the dense lists are executed (stalemate cleared), agents revert to the default RL navigator and resume tracking their original global waypoint lists. Full pseudo-code for PnR on the cropped subgrid, along with the completeness condition (two blanks per component) and the polynomial-time bound, are provided in Appendix~\ref{app:pnr}.


\noindent\textbf{4.5. Formal Guarantees for Triggered RL+MAPF Coordination: }We give formal guarantees for the \emph{trigger $\rightarrow$ locally confined MAPF $\rightarrow$ execution} cycle on the cropped subgrid $\mathcal{G}'$, as used by the Hybrid method. 
Let $\mathcal{A}_{\mathrm{lc}}(t)$ be the participant set at trigger time $t$, and let $\mathcal{I}'=\langle \mathcal{G}',\mathcal{A}_{\mathrm{lc}},S,T\rangle$ be the induced local MAPF instance; $S$ projects current poses to $V'$, and $T$ projects the \emph{active global waypoints} to $V'$.
PnR primitives, invariants, and theorems are stated in Appendix~\ref{app:pnr}.
Let $\Pi$ be the collision-free joint discrete plan returned by PnR on $\mathcal{G}'$ with horizon $L$.
For each $a\in\mathcal{A}_{\mathrm{lc}}$, build a dense waypoint list $\mathcal{W}^{\mathrm{dense}}_a$ by projecting the vertices of $\Pi$ to world-frame cell centers and collapsing consecutive waits.
During execution, the same decentralized RL policy tracks $\mathcal{W}^{\mathrm{dense}}_a$ under the standard speed bound.

\begin{theorem}[Finite-step deadlock clearance]
\label{thm:finite_clearance_main}
Assume that every connected component $C$ of $\mathcal{G}'$ contains at least two blanks ($b(C)\ge 2$) and that the induced instance $\mathcal{I}'$ is solvable. 
Then PnR returns $\Pi$; moreover, after tracking $\{\mathcal{W}^{\mathrm{dense}}_a\}$, every $a\in\mathcal{A}_{\mathrm{lc}}$ increases its active global \emph{waypoint index} by at least one in a finite number of simulator steps. 
Hence the local stalemate is cleared and long-range progress resumes for all participants.
\end{theorem}

\begin{proof}
By Appendix~\ref{app:pnr}, Theorem~\ref{thm:pnr-completeness}, PnR is complete on $\mathcal{G}'$ under $b(C)\ge2$ and returns a collision-free plan $\Pi$ whenever $\mathcal{I}'$ is solvable (cf.\ \citealp{dewilde2014pnr}). 
By construction (\S.4.4), each local target $T[a]$ is the grid projection of $a$'s \emph{active global waypoint} region. 
Mapping $\Pi$ to $\mathcal{W}^{\mathrm{dense}}_a$ preserves vertex order and yields a finite sequence. 
The decentralized RL tracker advances along $\mathcal{W}^{\mathrm{dense}}_a$ and enters the cell of $T[a]$ within finitely many simulator ticks (bounded speed and cell diameter). 
The waypoint manager then increments the active waypoint index for $a$. 
Since this holds for every $a\in\mathcal{A}_{\mathrm{lc}}$, the stalemate is cleared for the whole group.
\end{proof}

\begin{theorem}[Polynomially bounded coordination overhead]
\label{thm:bounded_overhead_main}
Under the same precondition $b(C)\ge2$, the number of primitive operations emitted by PnR on $\mathcal{G}'$ is polynomial in $|V'|$ and $|\mathcal{A}_{\mathrm{lc}}|$ (Appendix~\ref{app:pnr}, Theorem~\ref{thm:pnr-polytime}). 
Consequently, the plan horizon $L$ and the total dense-waypoint length $\sum_{a\in\mathcal{A}_{\mathrm{lc}}} |\mathcal{W}^{\mathrm{dense}}_a|$ are $O(\mathrm{poly}(|V'|,|\mathcal{A}_{\mathrm{lc}}|))$, implying that the \emph{per-trigger duty cycle} (fraction of execution steps under coordination) is polynomially bounded by crop size and group size.
\end{theorem}

\begin{proof}
Appendix~\ref{app:pnr}, Theorem~\ref{thm:pnr-polytime} states that PnR runs in polynomial time and emits a polynomial number of primitive operations in $|V'|$ and $|\mathcal{A}_{\mathrm{lc}}|$ (cf.\ \citealp{dewilde2014pnr}). 
Each primitive contributes $O(1)$ appended vertices per agent when forming $\Pi$, thus $L=O(\mathrm{poly}(\cdot))$. 
The mapping $\Pi\mapsto\{\mathcal{W}^{\mathrm{dense}}_a\}$ collapses waits and never expands steps, so $\sum_a |\mathcal{W}^{\mathrm{dense}}_a|=O(L)$. 
Execution consumes at most a constant number of simulator ticks per dense waypoint (bounded speed, bounded cell diameter), hence the duty cycle per trigger is polynomially bounded as claimed.
\end{proof}

\medskip
\noindent\textit{Context.}
Our “trigger $\rightarrow$ locally confined MAPF $\rightarrow$ handover” construction follows the locally confined MAPF design for resolving doorway/corridor stand-offs (e.g., \citealp{dergachev2021distributed}) and uses PnR for fast, complete reordering on the crop while all non-participants remain in decentralized RL.

\section{Experiment}


\textbf{5.1. Simulation Setup:} We evaluate in two static 2D maps that stress long-range progress under geometric bottlenecks: a \textbf{doorway} (short bottleneck) and a \textbf{corridor} (long single-lane) scenario, both known to induce reciprocal stand-offs for purely reactive VO style controllers \citep{han2022rlrvo,fiorini1998vo,vandenberg2011orca,stern2019mapfsurvey}. Robots are discs moving in discrete time with kinematics and sensing as in Problem Formulation; all core simulator parameters are summarized in Appendix~\ref{app:sim-params}.

For long-range guidance we use A* algorithm to generate per-agent global waypoint lists on a grid aligned with the world frame. Local navigation is \emph{always} executed by the same trained RL policy (RL-RVO style features and shaping), which tracks either global waypoints (default) or dense temporary waypoints returned by the local MAPF when triggered.
We compare \textbf{RL-only} vs.\ \textbf{RL+MAPF (ours)} under identical policies and maps to isolate the effect of the triggered local coordination. All the training and evaluations are conducted on a single workstation with CPU Intel Core i9-13900KF and GPU NVIDIA GeForce RTX 4080. The RL policy is trained by using PPO as the on-policy optimizer \citep{schulman2017ppo}; other training details follow \citet{han2022rlrvo}.


\noindent\textbf{5.2. Methods Under Comparison:}
Our goal is \emph{not} to rank different local navigation models. Instead, we propose a
\textbf{hybrid wrapper}---deadlock detection $+$ locally confined MAPF---that can be applied to an
existing local navigator with minimal changes. Accordingly, we hold the base RL navigator
fixed and compare it \emph{with vs.\ without} the wrapper to isolate
the incremental effect of local coordination, in line with prior locally confined hybrids \citep{dergachev2021distributed}. We conducted the comparison as:

\begin{itemize}[left=0pt]
  \item \textbf{RL-only (Base).} The decentralized RL local navigator runs alone (no deadlock detector, no MAPF), following \citet{han2022rlrvo}.
  \item \textbf{Hybrid (Ours).} The \emph{same} RL navigator augmented with: (i) the deadlock detector; and (ii) a locally confined MAPF module that solves the
  cropped subproblem with \textbf{Push-and-Rotate (PnR)} \citep{dewilde2014pnr}.%
\end{itemize}




\noindent\textbf{5.3. Evaluation Protocol \& Metrics:} For each scenario (doorway, corridor), each agent count/density, and each method (RL-only vs.\ Hybrid), we run \(N=100\) independent episodes and report aggregated results.
An episode is declared \emph{timeout} if any agent has not reached its goal within \(T_{\max}=1000\) simulation steps. An episode is a \emph{success} iff \emph{all} agents reach their goals within \(T_{\max}\) with no collisions. To ensure fairness,all methods share identical environment and protocol.
The primary metric is \emph{success rate}, standard for MAPF/navigation evaluations where reaching all goals is the main objective. We do
not claim improvements in average time/path length across all settings; local coordination can
introduce bounded waiting or detours by design. 

\begin{table}[t]
  \centering
  \caption{Success rate (\%) over $N{=}100$ trials per setting.}
  \label{tab:sr-combined}
  \vspace{2pt}
  \begin{tabular}{l ccc ccc}
    \hline
     & \multicolumn{3}{c}{\textbf{Doorway}} & \multicolumn{3}{c}{\textbf{Corridor}} \\
    \textbf{Method} & \textbf{4} & \textbf{6} & \textbf{8} & \textbf{4} & \textbf{6} & \textbf{8} \\
    \hline
    RL-only      & 15.63\% & 47.95\% & 39.34\% & 8.82\% & 0.00\% & 0.00\% \\
    Hybrid (PnR) & 100.00\% & 98.44\% & 95.45\% & 100.00\% & 100.00\% & 100.00\% \\
    \hline
  \end{tabular}
\end{table}

\begin{figure*}[t]
  \centering
  \includegraphics[width=0.8\textwidth]{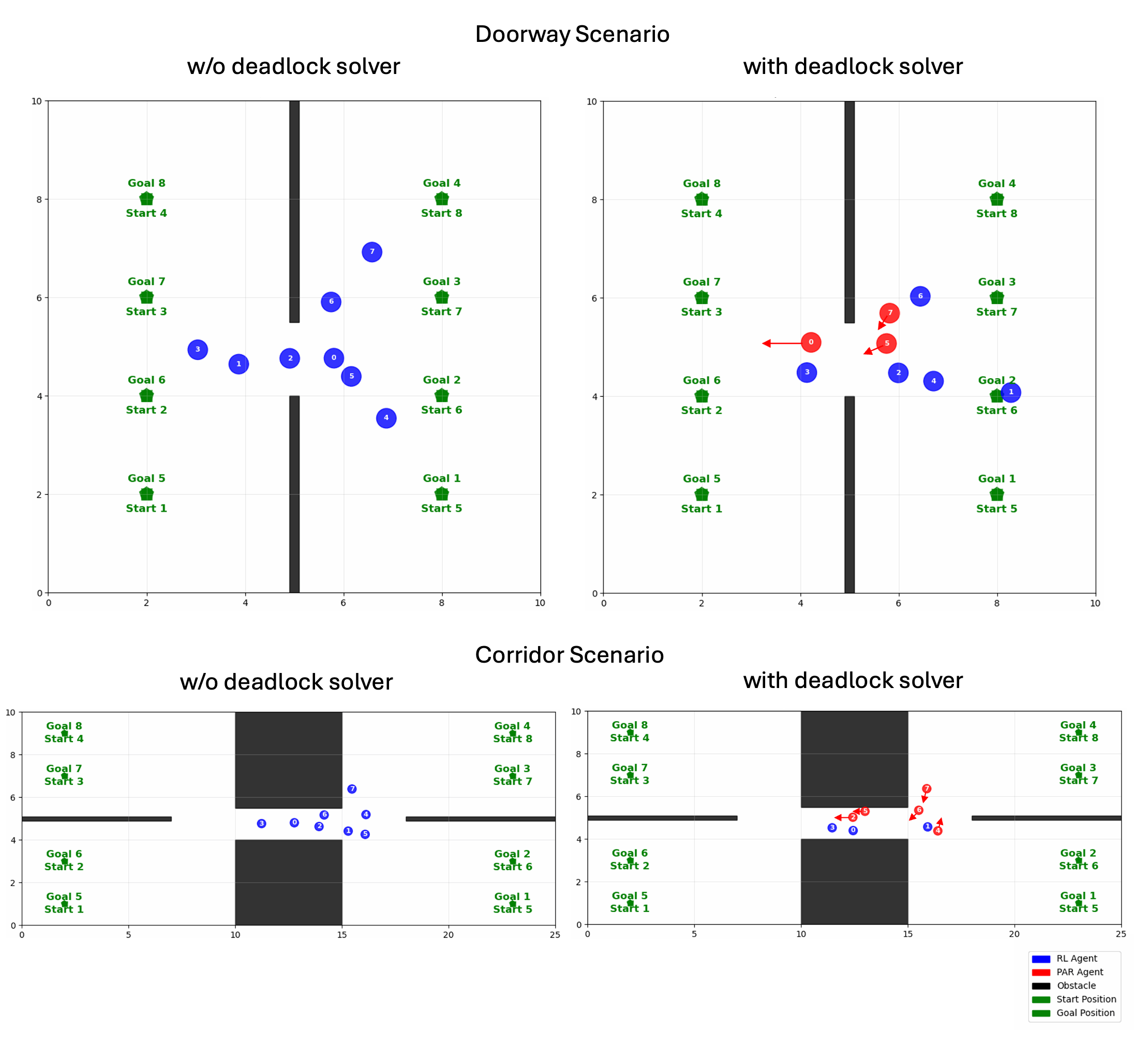}
  \caption{Qualitative comparison in \textbf{Doorway} (left two panels) and \textbf{Corridor} (right two panels). 
  }
  \label{fig:qual-door-corr}
\end{figure*}

\noindent\textbf{5.4. Result Analysis (Doorway Scenario): }
We use a single short bottleneck with doorway width \(1.5\,\mathrm{m}\).
Agent counts are \(\{4,6,8\}\). Starts/goals are \emph{fixed and symmetric} (left starts are right goals and vice versa).
For each agent count and each method (RL-only vs.\ Hybrid), we run \(N{=}100\) episodes with timeout \(T_{\max}{=}1000\) steps and declare success only if \emph{all} agents reach goals.
Table~\ref{tab:sr-combined} reports the success rate.
Hybrid (RL + deadlock trigger + local PnR) consistently improves completion in this short-bottleneck setting.
Figure~\ref{fig:qual-door-corr} (left panels) visualizes the effect: RL-only (left-most) settles into a reciprocal “yield–stall” in front of the doorway, whereas with the deadlock solver (second panel) a small red subset is temporarily coordinated to pass, restoring two-way flow while non-participants remain in RL. In short bottlenecks, \emph{minimal} coordination suffices to unlock progress, aligning with our on-demand local MAPF design.

\noindent\textbf{5.5. Result Analysis (Corridor Scenario): }
We evaluate a long single-lane corridor of length \(5\,\mathrm{m}\) and width \(1.5\,\mathrm{m}\) with two-way traffic.
Agent counts are \(\{4,6,8\}\) with the same fixed symmetric starts/goals.
We use the same protocol as the doorway case: \(N{=}100\), \(T_{\max}{=}1000\), success iff all agents reach goals.
Table~\ref{tab:sr-combined} summarizes the rates: RL-only degrades sharply with density, while the Hybrid maintains high completion via occasional PnR interventions.
Figure~\ref{fig:qual-door-corr} (right panels) shows why: the long bottleneck induces a persistent impasse under RL-only (third panel), but the triggered local MAPF (fourth panel, red subset) produces a feasible passing order inside the cropped subgrid and clears the stalemate; the rest of the team continues under decentralized RL.
This matches our goal: invoke coordination only where/when re-ordering is necessary, leaving the remainder to RL.





\section{Conclusions}
We proposed a hybrid, model-agnostic navigation framework that integrates decentralized RL with on-demand, locally bounded MAPF to resolve topological deadlocks. The wrapper intervenes only when progress stalls, applying Push-and-Rotate on a cropped subregion before returning control to RL. In doorway and corridor scenarios, the method consistently turned stalled interactions into successful completions while keeping coordination local. Future work includes adaptive deadlock detection, dynamic subregion sizing, and solver selection to further balance success and efficiency.

\bibliography{refs}

\appendix

\appendix

\section{Locally Confined Push-and-Rotate (PnR) on a Cropped Subgrid}
\label{app:pnr}

This appendix formalizes the primitives and guarantees of the
\emph{locally confined} Push-and-Rotate (PnR) procedure used by our deadlock
solver. PnR is complete on graphs whose every connected component has at least
two unoccupied vertices (``blanks'') and admits polynomial running time bounds
\citep{dewilde2014pnr,dewilde2013aamas}. We restate the setting on the cropped
subgrid $\mathcal{G}'$ and then give formal definitions for the primitives
(\emph{push, swap, rotate}) followed by complete proofs of the two theorems.

\subsection{Problem restatement on the cropped subgrid}
Given a cropped subgrid $\mathcal{G}'=(V',E')$ around the implicated agents
$\mathcal{A}_{\mathrm{lc}}$, we define a local MAPF instance
$\mathcal{I}'=\langle \mathcal{G}', \mathcal{A}_{\mathrm{lc}}, S, T\rangle$,
where $S:\mathcal{A}_{\mathrm{lc}}\!\to V'$ and $T:\mathcal{A}_{\mathrm{lc}}\!\to V'$ are the
start and (local) target assignments obtained by projecting current poses and
active global waypoints into $\mathcal{G}'$. 
Time is discretized by unit steps; at most one agent occupies a vertex at
any time; moves are to adjacent vertices or a wait action.

A (partial) configuration is a function $A:V'\to \mathcal{A}_{\mathrm{lc}}\cup\{\bot\}$,
where $A(v)=\bot$ denotes a blank. For a connected component $C$ of $\mathcal{G}'$,
let $b(C):=|\{v\in C:A(v)=\bot\}|$ be its blank count. We maintain a set $F\subseteq\mathcal{A}_{\mathrm{lc}}$
of \emph{finished} agents frozen at their targets.

\subsection{PnR primitives: formal definitions}
We write $A\Rightarrow A'$ for a single discrete-time transition of the
configuration (possibly moving multiple non-conflicting agents simultaneously).
Below, $P=(v_0,\ldots,v_k)$ denotes a simple path in $\mathcal{G}'$.

\paragraph{Push.}
Let $r\in\mathcal{A}_{\mathrm{lc}}\setminus F$ be the active agent and
$P=(v_0,\ldots,v_k)$ a shortest path from $A^{-1}(r)=v_0$ to $T[r]$ in
$\mathcal{G}'\setminus A[F]$. Suppose the next vertex $v_1$ is occupied by a
\emph{blocker} $b\notin F$ and there exists a simple path
$Q=(u_0,\ldots,u_\ell)$ inside the same connected component $C$ with
$u_0=v_1$, $u_\ell$ blank, and all $u_j$ distinct such that $b(C)\ge 2$ holds
throughout.\footnote{If needed, a short \emph{clear} step is taken to create a
temporary blank away from $A[F]$, as in \citet{dewilde2014pnr}.}
A \emph{push} of $b$ along $Q$ is the sequence of transitions that shifts the
blank from $u_\ell$ back to $u_0=v_1$:
\[
A(u_j)=\begin{cases}
\bot & j=\ell\\
x_{j+1} & 0\le j<\ell
\end{cases}
\quad\leadsto\quad
A'(u_j)=\begin{cases}
x_j & 1\le j\le \ell\\
\bot & j=0
\end{cases}
\]
where $x_j=A(u_j)$ are the pre-push occupants. After the push, $v_1$ becomes
blank, enabling $r$ to advance to $v_1$ on the next step. No vertex in $A[F]$
is modified.

\paragraph{Swap (within a biconnected block).}
Let $u,v$ be adjacent vertices in the same biconnected component (block)
$B\subseteq C$ with $A(u)=r$, $A(v)=b\notin F$, and $b(B)\ge 1$ (there exists
a blank elsewhere in $B$). A \emph{swap} is a constant-size sequence of moves
inside $B$ that exchanges the positions of $r$ and $b$ while leaving all other
agents unchanged and preserving $b(C)\ge 2$. Existence follows from the
presence of a cycle in $B$ and one blank, allowing a 3–4 move exchange
on a simple cycle (standard \emph{15-puzzle}–style maneuver) \citep{luna2011pushswap}.

\paragraph{Rotate (on a simple cycle).}
Let $Z=(z_0,\ldots,z_{k-1},z_k=z_0)$ be a simple cycle in a component $C$
with at least one blank on $Z$. A \emph{rotate} moves each agent on $Z$ to the
next vertex along the cycle in the same direction, filling the blank; since
$C$ has $b(C)\ge 2$, at least one blank remains off-cycle, preserving the
two-blank invariant. Rotation never touches $A[F]$ and keeps agents inside the
same block \citep{dewilde2014pnr}.

\subsection{Algorithmic scheme}
We instantiate PnR on $\mathcal{G}'$ as in Alg.~\ref{alg:pnr-appendix}: route
one active agent along a shortest path that avoids finished agents; when encountering
conflicts, resolve them by \emph{push}/\emph{swap}/\emph{rotate}. The joint plan $\Pi$
is later converted into a short \emph{dense waypoint list} for each participant and
tracked by the same RL low-level controller.

\begin{algorithm}[t]
\caption{Locally confined Push-and-Rotate (PnR) on cropped subgrid {\color{red} Fix the Alg.} $\mathcal{G}'$}
\label{alg:pnr-appendix}
\DontPrintSemicolon
\SetKwInOut{Input}{Input}\SetKwInOut{Output}{Output}
\Input{$\mathcal{G}'=(V',E')$; participants $\mathcal{A}_{\mathrm{lc}}$; starts $S$; targets $T$}
\Output{Joint collision-free discrete plan $\Pi$ over $\mathcal{G}'$}
\BlankLine
$A \leftarrow S$ \tcp*[r]{current assignment (at most one agent per vertex)}
$F \leftarrow \varnothing$ \tcp*[r]{finished set (agents fixed at targets)}
\While{$F \neq \mathcal{A}_{\mathrm{lc}}$}{
  Choose $r \in \mathcal{A}_{\mathrm{lc}}\setminus F$ (e.g., by ID or heuristic)\;
  Compute a shortest path $p$ from $A^{-1}(r)$ to $T[r]$ in $\mathcal{G}' \setminus A[F]$\;
  \While{$A^{-1}(r) \neq T[r]$}{
    $u \gets$ next vertex on $p$\;
    \uIf{$u$ is blank}{
      move $r \rightarrow u$; append to $\Pi$; update $A$\;
    }
    \uElseIf{$p$ closes a simple cycle in the current block}{
      \textsc{Rotate} along that cycle; append moves; update $A$\;
    }
    \uElseIf{there exists a blocker $b \notin F$ at $u$}{
      \textsc{Push} $b$ to a nearby blank (preserving two blanks per component and avoiding $A[F]$);\;
      append moves; update $A$\;
    }
    \Else{
      \textsc{Swap} $r$ with an adjacent blocker in the same block; append moves; update $A$\;
    }
  }
  $F \leftarrow F \cup \{r\}$ \tcp*[r]{freeze $r$ at $T[r]$}
}
\Return{$\Pi$}
\end{algorithm}

\subsection{Invariants and potential function}
We maintain the following invariants at every step:
\begin{enumerate}[label=(I\arabic*)]
\item \textbf{Finished fixed:} $A[F]=T[F]$ and finished vertices are never modified.
\item \textbf{Two blanks per component:} $b(C)\ge 2$ for every connected component $C$.
\item \textbf{Active simple path:} for the active agent $r$, there exists a simple path
      $p_r$ from $A^{-1}(r)$ to $T[r]$ in $\mathcal{G}'\setminus A[F]$.
\end{enumerate}
Define the lexicographic potential
\[
\Phi(A,r)\;=\;\big(\;D_r(A),\;K_r(A),\;H_r(A)\;\big)\in\mathbb{N}^3,
\]
where (i) $D_r$ is the graph distance from $A^{-1}(r)$ to $T[r]$ in
$\mathcal{G}'\setminus A[F]$, (ii) $K_r$ is the number of non-finished blockers
present on the chosen shortest path $p_r$, and (iii) $H_r$ is the sum of
shortest-path distances from each blocker on $p_r$ to the nearest blank within
its component (computed in $\mathcal{G}'\setminus A[F]$). We order potentials
lexicographically, i.e., $(a,b,c)\prec(a',b',c')$ iff $a<a'$ or ($a=a'$ and $b<b'$)
or ($a=a',b=b'$ and $c<c'$).

\begin{lemma}[Local progress]
\label{lem:local-progress}
Under invariants (I1)–(I3), any of the primitives (\emph{push}, \emph{swap}, \emph{rotate})
preserves (I1)–(I3) and yields $\Phi(A',r)\prec \Phi(A,r)$, where $A\Rightarrow A'$ is the
resulting configuration (and $r$ remains the active agent unless it moves to its target).
\end{lemma}

\begin{proof}
(I1) holds since primitives explicitly avoid $A[F]$. (I2) holds because each primitive
either moves an existing blank within a component (push), uses a cycle with at least one
off-cycle blank guaranteed by $b(C)\ge 2$ (rotate), or executes a constant-size exchange
within a biconnected block without consuming blanks (swap) \citep{dewilde2014pnr,luna2011pushswap}.
For $\Phi$: a blank delivered to the next vertex on $p_r$ either lets $r$ advance one step
(\(D_r\) strictly decreases), or reduces the count of blockers \(K_r\); when neither applies,
the relocation of a blocker strictly reduces its distance to a blank, decreasing \(H_r\).
Rotation on a cycle containing a blocker advances the blank past that blocker, decreasing
\(K_r\) or \(H_r\). Swap decreases \(K_r\) directly by exchanging \(r\) with the front blocker.
Thus $\Phi$ strictly decreases in lexicographic order.
\end{proof}

\subsection{Main guarantees}
\begin{theorem}[Completeness with two blanks per component]
\label{thm:pnr-completeness}
Let $\mathcal{G}'$ be any graph in which every connected component $C$ has at least
two blanks, i.e., $b(C)\ge 2$ for all $C$. Then for any solvable instance
$\mathcal{I}'=\langle \mathcal{G}',\mathcal{A}_{\mathrm{lc}},S,T\rangle$,
PnR returns a collision-free plan $\Pi$ that solves $\mathcal{I}'$.
\end{theorem}

\begin{proof}
Initiate with $F=\varnothing$ and pick any $r\in \mathcal{A}_{\mathrm{lc}}\setminus F$.
By solvability and (I2), a path $p_r$ exists in $\mathcal{G}'\setminus A[F]$, so (I3) holds.
By Lemma~\ref{lem:local-progress}, every primitive preserves invariants and strictly
decreases $\Phi(A,r)$ until either (a) $r$ reaches $T[r]$ or (b) another agent becomes
the active agent with smaller $D_r$ at its turn; in both cases, termination occurs in
finitely many steps because $\Phi\in\mathbb{N}^3$ is well-founded. When $r$ reaches
$T[r]$, we freeze it: $F\leftarrow F\cup\{r\}$, which preserves (I1)–(I2). Repeating
this process assigns all agents to $T[\cdot]$ in finitely many phases (at most
$|\mathcal{A}_{\mathrm{lc}}|$). Since primitives never modify finished vertices and
moves are adjacent or wait, the concatenation yields a collision-free plan. Hence PnR
returns a valid solution whenever one exists.
\end{proof}

\begin{theorem}[Polynomial running time]
\label{thm:pnr-polytime}
Under the same precondition $b(C)\ge 2$ for all components $C$ of $\mathcal{G}'$,
PnR runs in time polynomial in $|V'|$ and $|\mathcal{A}_{\mathrm{lc}}|$.
\end{theorem}

\begin{proof}
As shown by \citet{dewilde2014pnr}, each agent becomes active a bounded number
of phases and, in each phase, the number of primitive operations (push/swap/rotate)
is bounded by a polynomial in $|V'|$ (each operation affects a constant-size pattern,
and path recomputations can be done in $O(|E'|)$ time). Summing over at most
$|\mathcal{A}_{\mathrm{lc}}|$ agents yields a polynomial bound on the total number of
operations and the running time in $|V'|$ and $|\mathcal{A}_{\mathrm{lc}}|$. Therefore,
PnR admits a polynomial-time complexity on $\mathcal{G}'$.
\end{proof}

\subsection{From discrete plan to execution}
The joint plan $\Pi$ is converted into a short \emph{dense waypoint list} for each
participant. During execution, the same decentralized RL policy tracks these dense
waypoints while handling reciprocal avoidance; upon completion, agents revert to
the default RL navigation and resume their original global waypoint lists.

\section{Simulator Parameters}
\label{app:sim-params}

\begingroup
\small
\setlength{\tabcolsep}{6pt}
\renewcommand{\arraystretch}{1.1}
\begin{table}[h]
  \centering
  \caption{Core simulator parameters used in all experiments.}
  \vspace{2pt}
  \begin{tabular}{l c c}
    \hline
    \textbf{Parameter} & \textbf{Symbol} & \textbf{Value} \\
    \hline
    Time step & \(\Delta t\) & \(0.1\,\mathrm{s}\) \\
    Control frequency & \(f_c\) & \(10\,\mathrm{Hz}\) \\
    Max speed & \(v_{\max}\) & \(1.5\,\mathrm{m/s}\) \\
    Agent radius & \(r\) & \(0.2\,\mathrm{m}\) \\
    Sensing horizon & \(R\) & \(5\,\mathrm{m}\) \\
    Neighbor cap (impl.) & \text{n/a} & \(10\) \\
    \hline
  \end{tabular}
\end{table}
\endgroup

\end{document}